\documentclass[twoside]{article}
\usepackage[accepted]{aistats2019}

\usepackage{graphicx}  
\usepackage[round]{natbib}
\usepackage[dvipsnames]{xcolor}

\usepackage[backref=page]{hyperref}
\hypersetup{
  citecolor = NavyBlue,
  colorlinks = true
}
\usepackage{backref}

\usepackage{subcaption}
\usepackage[switch,columnwise]{lineno}
\usepackage[lined, boxed, ruled, commentsnumbered, noend]{algorithm2e}
\usepackage{amssymb,amsmath,amsthm}
\usepackage{bbm}
\usepackage{xspace}

\SetCommentSty{mycommfont}

\newtheorem{theorem}{Theorem}

\newtheorem{lemma}[theorem]{Lemma}

\theoremstyle{definition}

\theoremstyle{remark}

\numberwithin{equation}{section}

\newcommand{\NonNegativeReals}{\ensuremath{\mathbb{R}_{\ge 0}}}

\newcommand{\integers}{\ensuremath{\mathbb{Z}}}

\newcommand{\reals}{\ensuremath{\mathbb{R}}}

\DeclareMathOperator*{\argmin}{arg\,min}
\DeclareMathOperator*{\argmax}{arg\,max}

\usepackage{dsfont}
\newcommand\id{\ensuremath{\mathds{1}}}

\newcommand{\expctover}[2]{\mathbb{E}_{#1}\!\left[#2\right]}

\newcommand{\abs}[1]{\left\vert#1\right\vert}
\newcommand{\given}{\mid}

\newcommand{\denselist}{\itemsep 0pt\topsep-6pt\partopsep-6pt}

\newcommand{\figref}[1]{Fig.~\ref{#1}}

\newcommand{\secref}[1]{\S\ref{#1}}

\newcommand{\lemref}[1]{Lemma~\ref{#1}}
\newcommand{\algref}[1]{Algorithm~\ref{#1}}

\newcommand{\lnref}[1]{Line~\ref{#1}}

\newcommand{\cA}{{\mathcal{A}}}
\newcommand{\cB}{{\mathcal{B}}}
\newcommand{\cC}{{\mathcal{C}}}

\newcommand{\cS}{{\mathcal{S}}}

\newcommand{\cX}{{\mathcal{X}}}

\newcommand{\bk}{{\mathbf{k}}}

\newcommand{\by}{{\mathbf{y}}}

\newcommand{\paren} [1] {\ensuremath{ \left( {#1} \right) }}

\newcommand{\curlybracket}[1]{\left\{#1\right\}}

\newcommand{\algname}{\textsf{Online-DSOpt}\xspace}

\newcommand{\alglbds}{\textsf{DSOpt-SA}\xspace}
\newcommand{\algdcds}{\textsf{DSOpt-DC}\xspace}

\newcommand{\alggreedy}{\textsf{Greedy}\xspace}

\newcommand{\algadd}{\textsf{Greedy-Add}\xspace}
\newcommand{\algrem}{\textsf{Greedy-Rem}\xspace}

\newcommand{\alginit}{\textsf{Init}\xspace}
\newcommand{\algonlinedso}{\textsf{Online-DSOpt}\xspace}
\newcommand{\algdsopt}{\textsf{DSOpt}\xspace}
\newcommand{\algmodmod}{\textsf{ModMod}\xspace}
\newcommand{\algsupsub}{\textsf{SupSub}\xspace}
\newcommand{\alglocalsearch}{\textsf{LocalSearch}\xspace}

\newcommand{\algdssa}{\textsf{DSConstruct-SA}\xspace}
\newcommand{\algdsdc}{\textsf{DSConstruct-DC}\xspace}

\newcommand{\algreward}{\textsf{ComputeReward}\xspace}

\newcommand{\utility}{f} 
\newcommand{\batchUtil}{F}
\newcommand{\batchUtilApprox}{\hat{\batchUtil}}

\newcommand{\ub}{\textsf{ub}}
\newcommand{\lb}{\textsf{lb}}

\newcommand{\dsf}{h}
\newcommand{\dsg}{g}
\newcommand{\ubdsf}{\ub^{\dsf}}
\newcommand{\lbdsg}{\lb^{\dsg}}

\newcommand{\cost}{c}

\newcommand{\GP}[1]{\text{GP}\paren{#1}} 
\newcommand{\mean}{\mu}

\newcommand{\cov}{k}
\newcommand{\bcov}{\bk}
\newcommand{\Cov}{K}
\newcommand{\normal}{\mathcal{N}}
\newcommand{\selected}{\cA} 

\newcommand{\constrDom}{\cC}
\newcommand{\constrSelected}{\cS}
\newcommand{\queryPool}{Q}

\newcommand{\constrSelectedAt}[1]{\constrSelected^{\paren{#1}}}
\newcommand{\constrDomAt}[1]{\constrDom^{\paren{#1}}}

\newcommand{\numpos}{L}
\newcommand{\pid}{\ell}
\newcommand{\varpid}{k}

\newcommand{\randBatch}{\cB_\queryPool}
\newcommand{\randState}{\phi}

\newcommand{\ex}{x} 
\newcommand{\obs}{y} 
\newcommand{\bobs}{\by} 

\newcommand{\exDom}{\cX}

\newcommand{\noise}{\varepsilon}

\newcommand{\batchBudget}{n}

\newcommand{\mulfun}{r}
\newcommand{\mulconvfun}{u}

\usepackage{etoolbox}

\newtoggle{longversion}
\settoggle{longversion}{true}
\iftoggle{longversion}{}
{}

\begin{document}
\twocolumn[

\aistatstitle{Batched Stochastic Bayesian Optimization via Combinatorial Constraints Design}

\addtocounter{footnote}{-1}

\aistatsauthor{ Kevin K. Yang\footnotemark \And Yuxin Chen \And Alycia Lee \And Yisong Yue }
\aistatsaddress{ FVL57 \And  Caltech \And Caltech \And Caltech} ]

\footnotetext{Research done while author was at Caltech.}

\begin{abstract}
 
In many high-throughput experimental design settings, such as those common in biochemical engineering, batched queries are more cost effective than one-by-one sequential queries. Furthermore, it is often not possible to directly choose items to query. Instead, the experimenter specifies a set of constraints that generates a library of possible items, which are then selected stochastically. Motivated by these considerations, we investigate \emph{Batched Stochastic Bayesian Optimization} (BSBO), a novel Bayesian optimization scheme for choosing the constraints in order to guide exploration towards items with greater utility. We focus on \emph{site-saturation mutagenesis}, a prototypical setting of BSBO in biochemical engineering, and propose a natural objective function for this problem. Importantly, we show that our objective function can be efficiently decomposed as a difference of submodular functions (DS), which allows us to employ DS optimization tools to greedily identify sets of constraints that increase the likelihood of finding items with high utility. Our experimental results show that our algorithm outperforms common heuristics on both synthetic and two real protein datasets.

 \end{abstract}


\section{Introduction}

Bayesian optimization is a popular technique for optimizing black-box objective functions, with applications in (sequential) experimental design,  parameter tuning, recommender systems and more. In the classical setting, Bayesian optimization techniques assume that items can be directly queried at each iteration. However, in many real-world applications such as those in biochemical engineering, direct querying is not possible: instead, a (constrained) library of items is specified, and then batches of items from the library are stochastically queried.

\begin{figure}[t]
  \centering
  \includegraphics[trim={0pt 0pt 0pt 0pt}, width=.45\textwidth]{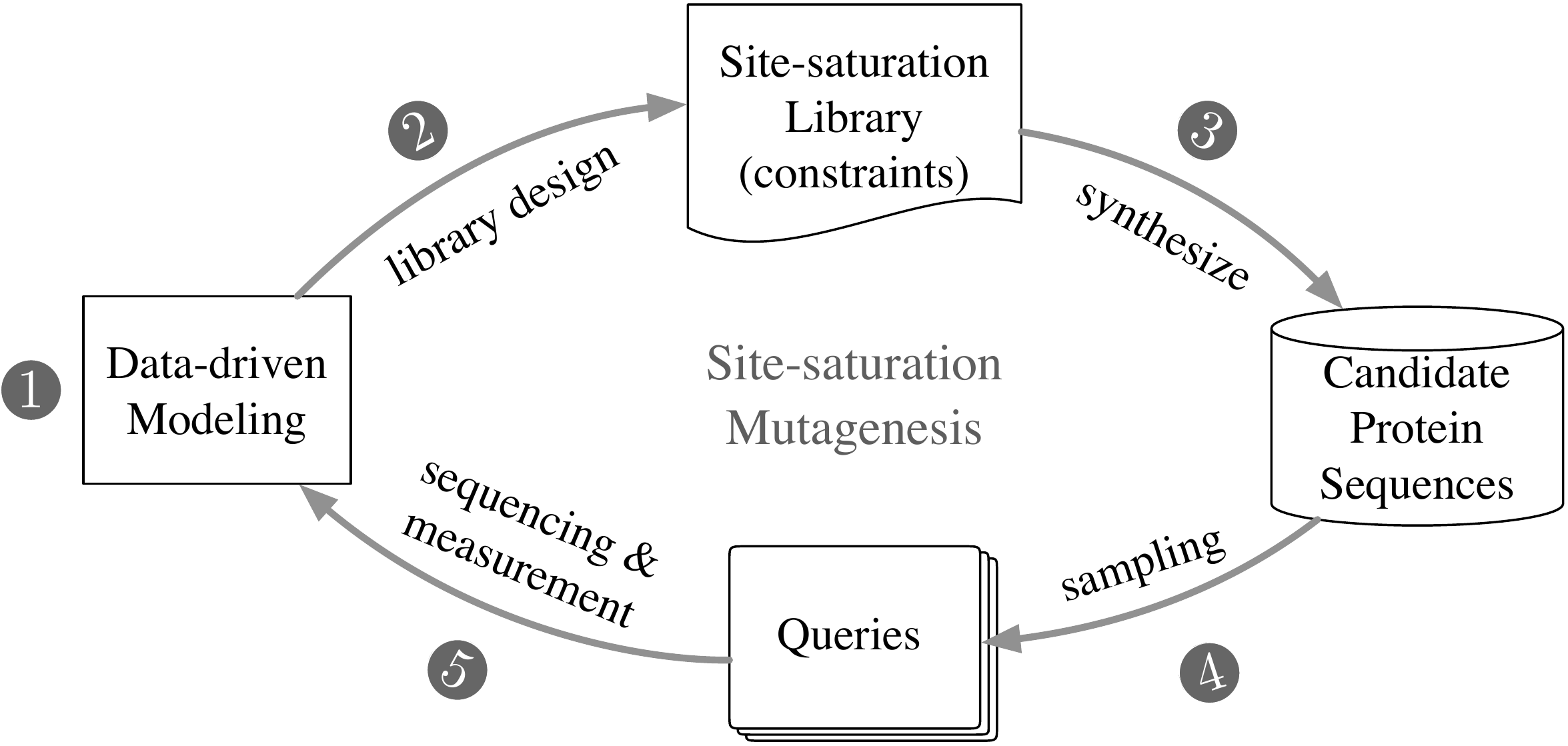}
  \caption{Data-driven site-saturation mutagenesis. (1) Machine learning model for predicting certain protein properties; (2) site-saturation library design; (3) synthesize protein sequences according to the site-saturation libraries; (4) randomly sample proteins for sequencing; (5) sequence and measure the properties of the sampled proteins.}
  \label{fig:alg:illustration}
  \vspace{-5mm}
\end{figure}

As a prototypical example in biochemical engineering, let us consider \emph{site-saturation mutagenesis} (SSM) \citep{voigt2001computationally}, a protein-engineering strategy that mutates a small number of critical sites in a protein sequence (cf. \figref{fig:alg:illustration}). At each round, a combinatorial library is designed by specifying which amino acids are allowed at the specified sites (step (1-3)), and then a batch of amino acid  sequences from the library is sampled with replacement (step 4). The sampled sequences are evaluated for their ability to perform a desired function (step 5), such as a chemical reaction. 

Ideally, at each iteration, the amino acids to be considered at each site should be chosen to maximize the number of improved sequences expected in the stochastic batch sample from the resulting library. Finding such libraries is highly non-trivial: it requires solving a combinatorial optimization problem over an exponential number of items. Libraries are designed by choosing the allowed amino acids at each site (`constraints') from the set of all amino acids at all sites. Adding allowed constraints results in an exponential number of items in the library. 
As mentioned above, these challenges are exacerbated due to the uncertainty from sampling batches of queries.
Thus, new optimization schemes and algorithmic tools are needed for addressing such problems.

\vspace{-2mm}
\paragraph{Our contribution} In this paper, 
we investigate \emph{Batched Stochastic Bayesian Optimization} (BSBO), a novel Bayesian optimization scheme for choosing a library design in order to guide exploration towards items with greater utility. This scheme is unique in that we choose a library design instead of directly querying items, and the items are queried in stochastic batches (e.g. 10-1000 items per batch). In particular, we focus on library design for site-saturation mutagenesis, and identify a natural objective function that evaluates the quality of a library design given the current information about the system. We propose \algname, an efficient online algorithm for optimization over stochastic batches. In a nutshell, \algname assembles each batch by decomposing the objective function into the difference of two submodular functions (DS). This allows us to employ DS optimization tools (e.g., \citet{narasimhan2005submodular}) to greedily identify sets of constraints that increase the likelihood of finding items with high utility. We demonstrate the performance of \algname on both synthetic and two experimentally-generated protein datasets, and show that our algorithm in general outperforms conventional greedy heuristics and efficiently finds rare, highly-improved, sequences.



\vspace{-2mm}
\section{Related Work}
\vspace{-2mm}
\paragraph{Bayesian optimization with Gaussian processes}
\looseness -1 Our work addresses a specific setting for Gaussian process (GP) optimization. GPs are infinite collections of random variables such that every finite subset of random variables has a multivariate Gaussian distribution.
A key advantage of GPs is that inference is very efficient, which makes them one of the most popular theoretical tools for Bayesian optimization
\citep{rasmussen:williams:2006, srinivas10gaussian, wang2016optimization}. 
 Notably, \citet{srinivas10gaussian} introduce the Gaussian Process Upper Confidence Bound (GP-UCB) algorithm for Bayesian Bandit optimization, which provides bounds on the cumulative regret when sequentially querying items. \citet{desautels2014parallelizing} generalize this to batch queries. In contrast to our setting, these algorithms require the ability to \emph{directly} query items, either sequentially or in batches. 
 
 \vspace{-2mm}
\paragraph{GP optimization for protein engineering} GP-UCB has been used to find improved protein sequences when sequences can be queried directly \citep{romero2013navigating,bedbrook2017machine}. GPs \citep{Saito2018MachineLearningGuidedMF} and other machine-learning methods \citep{wu2019machine} have been used to select constraints for SSM libraries. However, previous work relied on ad-hoc heuristics and do not provide a general procedure for selecting constraints in a model-driven way. 

\vspace{-2mm}
\paragraph{Information-parallel learning}
\looseness -1 In addition to the bandit setting \citep{desautels2014parallelizing}, there is a large body of literature on various machine learning settings 
that exploit information-parallelism. For example, in large-scale optimization, mini-batch/parallel training has been extensively explored to reduce the training time of stochastic gradient descent \citep{li2014efficient, zinkevich2010parallelized}. 
In batch-mode active learning \citep{Hoi2006BMA, guillory2010interactive,chen13near}, an active learner selects a set of examples to be labeled simultaneously. The motivation behind batch active learning is that in some cases it is more cost-effective to request labels in large batches, rather than one-at-a-time. This setting is also referred to as buy-in-bulk learning \citep{yang2013buy}. In addition to the simpler modeling assumption of being able to directly issue queries, these approaches also differ from our setting in terms of the objective: the batch-mode active learning algorithms aim to find a set of items that are maximally informative about some target hypothesis (hence to maximally explore), whereas we want to identify the best item (i.e., to both explore and exploit).

\vspace{-2mm}
\paragraph{Submodularity and DS optimization} 
Submodularity \citep{nemhauser1978analysis} is a key tool for solving many discrete optimization problems, and has been widely recognized in recent years in theoretical computer science and machine learning. While a growing number of previously studied problems can be expressed as submodular minimization \citep{jegelka2011submodularity} or maximization \citep{kempe2003maximizing, krause2007near} problems, standard maximization and minimization formulations only capture a small subset of discrete optimization problems. \citet{narasimhan2005submodular} show that any set function $q$ can be decomposed as the difference of two submodular functions $\dsf$ and $\dsg$. Replacing $\dsf$ with its modular upper bound, $\dsg$ with its modular lower bound, or both reduces the problem of minimizing $q$ to a series of submodular minimizations, submodular maximizations, or modular minimizations, respectively, that are guaranteed to reduce $q$ at every iteration and to arrive at a local minimum of $q$ \citep{iyer2012algorithms}. In general, computing a DS decomposition requires exponential time. We present two polynomial-time decompositions of our objective function.


\section{Problem Statement}
\subsection{Problem Setup}
We aim to optimize a black box utility function, $\utility: \exDom \rightarrow \reals$. In contrast to classical Bayesian optimization, which sequentially queries the function value $\utility(\ex_t)$ for an selected item $\ex_t \in \exDom$, we assume that the experimenter can only choose a subset of constraints (i.e., rules for generating items) from a ground set $\constrDom$, based on which a stochastic batch of items are generated and measured. More concretely, we consider the following interactive protocol, as illustrated in \figref{fig:bsbo-flowchart}. At each round the following happens: 
\begin{itemize}\denselist
    \item The algorithm chooses a set of constraints $\constrSelected \subseteq \constrDom$ based on current knowledge of $\utility$ (\figref{fig:bsbo-flowchart}, step (2)). 
    \item The chosen constraints are used to construct a \emph{library} of candidate queries: $\queryPool(\constrSelected) \subseteq \exDom$, where $\queryPool: 2^\constrDom \rightarrow 2^\exDom$ denotes the physical process that produces items under these constraints (\figref{fig:bsbo-flowchart}, step (3)).
    \item A batch of $\batchBudget$ queries $\randBatch(\constrSelected, \randState) \subseteq \exDom$ is randomly selected from the library $\queryPool(\constrSelected)$ 
    via a stochastic sampling procedure 
    (\figref{fig:bsbo-flowchart}, step (4)). 
    Here, $\randState$ represents the random state of the sampling procedure $\randBatch$.
    \item 
    When querying each item  $\ex \in \randBatch(\constrSelected, \randState)$, we observe the function value there, perturbed by noise: $\obs = \utility(\ex) + \noise(\ex)$. Here, the noise $\noise(\ex)\sim \normal(0, \sigma^2)$ represents i.i.d. Gaussian white noise. We further assume that querying each item  $\ex$ achieves reward $\utility(\ex)$ and incurs some cost $\cost(\{\ex\})$, where $\cost: 2^\exDom \rightarrow \reals$ denotes the cost function of a set of items (\figref{fig:bsbo-flowchart}, step (5)). 
    \item The results of the queries are used to update 
    $\utility$.
\end{itemize}
\begin{figure}
  \centering
  \includegraphics[trim={0pt 0pt 0pt 0pt}, width=.45\textwidth]{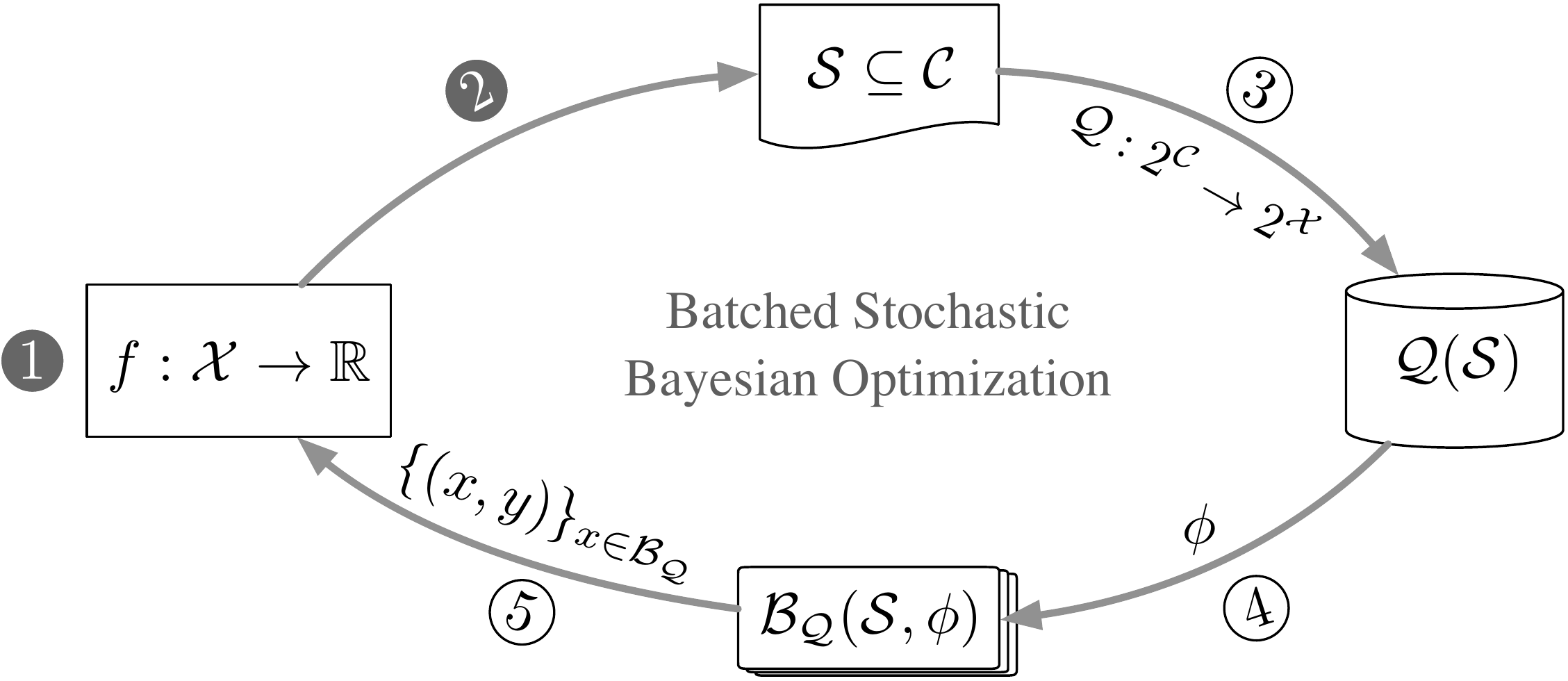}
  \caption{The batched stochastic Bayesian optimization setting. (1) Bayesian modeling (2) combinatorial constraints design; (3) candidate query generation; (4) random sampling; (5) batched queries.}
  \label{fig:bsbo-flowchart}
\end{figure}

We model $\utility$ as a sample from a Gaussian process, denoted by $\utility \sim \GP{\mu(x), \cov(\ex, \ex')}$. Suppose that we have queried $\selected \subseteq \exDom$ and received $\bobs_\selected = [\utility(\ex_i) + \noise(\ex_i)]_{\ex_i\in\selected}$ observations. We can obtain the posterior mean $\mean_\selected(\ex)$ and covariance $\cov_\selected{(\ex, \ex')}$ of the function through the covariance matrix $\Cov_\selected = [\cov(\ex_i, \ex_j) + \sigma^2 \id]_{\ex_i, \ex_j \in \selected}$ and $\bcov_\selected(\ex) = [\cov(\ex_i, \ex)]_{\ex_i \in \selected}$:
\begin{align*}
\mean_\selected(\ex) &= \mean(\ex) + \bcov_\selected(\ex)^\intercal\Cov_\selected^{-1} \bobs_\selected,\\
\cov_\selected{(\ex, \ex')} &= \cov(\ex, \ex') - \bcov_\selected(\ex)^\intercal\Cov_\selected^{-1} \bcov_\selected(\ex').
\end{align*}

\subsection{The Objective}\label{sec:obj}
\paragraph{Simple regret} 
Our overall goal is to minimize the (expected) simple regret, defined as 
$R_t(\ex) = \utility(\ex^*) - \utility(\ex_t)$ over $T$ rounds, where $\ex^* = \argmax_{\ex\in\exDom}\utility(\ex)$ is the item of the maximum utility. In other words, we aim to  maximize the reward in order to converge to performing
as well as $\ex^*$ 
as efficiently as possible. We refer to this problem as the batched stochastic Bayesian optimization (BSBO) problem.\footnote{When the set of candidate queries $\queryPool(\constrSelected_t)$ contains only \emph{one} unique item at each round $t$, then the BSBO problem reduces to the standard Bayesian optimization problem.}

\paragraph{Acquisition function} Assume that we have a budget of querying $\batchBudget$ items for each batch of experiments and that each batch $\randBatch$ is selected by sampling uniformly from the library. At each iteration, we wish to select the constraints $\constrSelected$ that will maximize the (expected) number of improved items observed in the next stochastic batched query. If the current best item has a value $\tau$, then we seek a set of constraints $\constrSelected^* \in \argmax \batchUtil(\constrSelected)$, where:
\begin{align}
    \label{eq:batch-util}
    \batchUtil(\constrSelected) = \expctover{\randState}{\sum_{x \in \randBatch(\constrSelected)}\mathbbm{1}(\utility(\ex) > \tau)}.
\end{align}
Here, $\mathbbm{1}$ is the indicator function. This objective is intractable under the GP posterior, as the dependencies between $\utility(\ex)$ preclude a closed form. We ignore the dependencies between the utilities to arrive at the following surrogate function;
\begin{align}
  \label{eq:batch-util-approx}
  \batchUtilApprox(\constrSelected) = \sum_{\ex \in \queryPool(\constrSelected)} \rho(\ex) \left[1-\paren{1-\frac{1}{|\queryPool(\constrSelected)|} }^{\batchBudget}\right].
\end{align}
The rewards $\rho(\ex) = P(\utility(\ex) > \tau)$ can be computed for all $\ex \in \queryPool(\constrDom)$ from the GP posterior for each item using the Gaussian survival function by ignoring off-diagonal entries in the predictive posterior covariance. 
Note that this surrogate acquisition function $\batchUtilApprox$ captures the expected reward under an \emph{independence} assumption. As we will demonstrate later in \secref{sec:exp:real}, despite such an assumption, we observe a strong correlation between $\batchUtilApprox$ and $\batchUtil$ on the experimental datasets we study, which are known to have high dependencies between the rewards $\rho(x)$ of different items.

\subsection{Site-Saturation Library Design}
We now consider the site-saturation library design problem as a special case of BSBO. In site-saturation mutagenesis, the utility function $\utility(x)$ specifies the utility of a protein sequence $x$, and the constraint set $\constrDom = \bigcup_{\pid=1}^\numpos \constrDomAt{\pid}$ specifies the set of amino acids allowed at each site of the protein sequence. $\numpos$ denotes the number of sites, and $\constrDomAt{\pid}$ denotes the set of all possible amino acids\footnote{$\constrDomAt{\pid}$ is typically the 20 canonical amino acids.} at site $\pid$. We denote the set of amino acids selected for site $\pid$ by $\constrSelectedAt{\pid}$; 
hence $\constrSelected = \bigcup_{\pid=1}^{\numpos} \constrSelectedAt{\pid}$. The candidate query pool (library) $\queryPool$ consists of all possible protein sequences that can be generated w.r.t. the constraints: $\constrSelected$:
\begin{align}\label{eq:querypool}
  \queryPool\paren{\constrSelected} = \prod_{\pid=1}^\numpos \constrSelectedAt{\pid}.
\end{align}
Note that adding constraints generally \emph{increases} the number of allowed items. 

\begin{algorithm}[t]
  \nl {\bf Input}: 
  Constraints set $\constrDom = \bigcup_{\pid=1}^\numpos \constrDomAt{\pid}$; number of rounds $T$; budget on each batch $n$; GP prior on $\utility$\\ 
  \Begin{
    \nl $\selected \leftarrow \emptyset$ \\
    \tcc{iteratively select the next batch}
    \For{$t$ in ${1, \dots, T}$} 
    {
      \tcc{compute the reward matrix $M=\{f(\ex)\}$}
      \nl $M \leftarrow \algreward(\text{posterior on~} \utility, \selected)$ \label{ln:alg:ol:computereward}\\
      \nl $\constrSelected \leftarrow \algdsopt(\constrDom, M, n)$\\
      \tcc{posterior update}
      \nl $\selected \leftarrow \selected \cup \randBatch(\constrSelected, \randState)$
    }
    \nl {\bf Output}: Optimizer of $\utility$ \\
  }
  \caption{Online Batched Constraints Design via DS Optimization (\algname)}\label{alg:online-dso}
\end{algorithm}

\section{Algorithms}


\begin{algorithm}[t]
  \nl {\bf Input}: Constraints set $\constrDom = \bigcup_{\pid=1}^\numpos \constrDomAt{\pid}$; reward matrix $M=\{f(x)\}$; budget on each batch $n$\\ 
  \Begin{
    \nl Set up $\batchUtilApprox$ from the inputs $(M, n)$ \\
    \tcc{Decompose $\batchUtilApprox$ into diff of submod funcs.}
    \nl $\dsf, \dsg \leftarrow \algdssa(\batchUtilApprox, \constrDom)$ \\
    (or $\dsf, \dsg \leftarrow \algdsdc(\batchUtilApprox, \constrDom)$)\\
    \tcc{Initiliaze the starting position}
    \nl $\constrSelected_{\text{cand}} \leftarrow \emptyset$ \\
    \nl $\constrSelected \leftarrow \alginit(\constrDom)$\\
    \tcc{Optimize $\dsf-\dsg$ using \algmodmod or \algsupsub.}
    \While{$\constrSelected$ not converged} 
    {
      \tcc{Keep track of local search solutions}
      \nl $\constrSelected_{\text{cand}} \leftarrow \constrSelected_{\text{cand}} \cup \alglocalsearch(\batchUtilApprox, \constrSelected, \constrDom)$ \label{ln:dsopt:locsearch} \\
      \tcc{Make a greedy move from $\constrSelected$}
      \nl $\constrSelected \leftarrow \algmodmod(\dsf-\dsg, \constrSelected, \constrDom)$ \label{ln:dsopt:modmod}\\
      (or $\constrSelected \leftarrow \algsupsub(\dsf-\dsg, \constrSelected, \constrDom)$) \label{ln:dsopt:supsub} \\
    }
    \nl $\constrSelected_{\text{cand}} \leftarrow \constrSelected_{\text{cand}} \cup \{\constrSelected\}$ \\
    \tcc{pick the best among candidate solutions}
    \nl $\constrSelected^* \leftarrow \argmin_{\constrSelected \in \constrSelected_{\text{cand}}}{\curlybracket{\batchUtilApprox(\constrSelected)}}$ \\
    \nl {\bf Output}: Set of selected constraints $\constrSelected^*$ \\
  }
  \caption{DS Optimization (\algdsopt)}\label{alg:dso}
\end{algorithm}

We now present \algonlinedso (\algref{alg:online-dso}), an  online learning framework for (online) batched stochastic Bayesian optimization. Our framework relies on a novel discrete optimization subroutine, \algdsopt, which aims to maximize the expected reward for each batched experiment. At each iteration, \algonlinedso uses a GP trained on previously-observed items to compute the reward for each item $\ex\in\queryPool(\constrDom)$ (cf., \lnref{ln:alg:ol:computereward}) and then invokes \algdsopt to select constraints. Pseudocode for \algdsopt is presented in \algref{alg:dso}. A batch of items is then sampled stochastically from the resulting library and used to update the GP. 

A key component of the DS optimization subroutine \algdsopt is a DS decomposition of the objective. Note that in general, finding a DS decomposition of an arbitrary set function requires searching through a combinatorial space and can be computationally prohibitive. As one of our main contributions, we present two polynomial time algorithms, \algdssa (\algref{alg:lb-ds}) and \algdsdc (\algref{alg:dc-ds}), for decomposing our surrogate objective. Both algorithms exploit the structure of the objective function: \algdssa decomposes the objective via submodular augmentation \citep{narasimhan2005submodular}; \algdsdc decomposes the objective via a difference of convex functions (DC) decomposition.

\subsection{DS Optimization}
After obtaining the submodular decomposition $\batchUtilApprox=-(\dsf-\dsg)$, \algdsopt (\algref{alg:dso}) proceeds to greedily optimize the DS function. For example, let us consider running the Modular-modular procedure (ModMod) \citep{iyer2012algorithms} for making a greedy move at \lnref{ln:dsopt:modmod} of \algref{alg:dso}. Since our goal is to maximize $-(\dsf-\dsg)$ (i.e., to minimize $\dsf-\dsg$), we will seek to minimize the upper bound on $\dsf-\dsg$. The ModMod procedure constructs a modular upper bound on the first submodular component, denoted by $\ubdsf \geq \dsf$, and a modular lower bound on the second submodular component, denoted by $\lbdsg \leq \dsg$. Both modular bounds are tight at the current solution $\constrSelected$: $\ubdsf(\constrSelected) = \dsf(\constrSelected)$, $\lbdsg(\constrSelected) =\dsg(\constrSelected)$. ModMod then tries to solve the following optimization problem, starting from $\constrSelected$: 
\begin{align*}
  \constrSelected^* \in \argmin_{\constrSelected}   \paren{\ubdsf(\constrSelected) - \lbdsg(\constrSelected)}.
\end{align*}
To ensure that we find a better solution, we augment the ModMod procedure with a sequence of additional local search solutions, and in the end pick the best among all. The local search procedure, \alglocalsearch (cf. \lnref{ln:dsopt:locsearch} of \algref{alg:dso}), sequentially makes greedy steps (by adding or removing a constraint from the current solution) until no further action is improving the current solution. The following theorem states that our DS optimization subroutine \algdsopt is guaranteed to find a ``good'' solution: 
\begin{theorem}[Adapted from \cite{iyer2012algorithms}]
  \algref{alg:dso} is guaranteed to find a set of constraints that achieves a local maximum of $\batchUtilApprox$.
\end{theorem}

\begin{algorithm}[t]
  \nl {\bf Input}: Constraints set $\constrDom = \bigcup_{\pid=1}^\numpos \constrDomAt{\pid}$; surrogate objective function $\batchUtilApprox$, budget on each batch $n$; selected constraints $\constrSelected$ \\ 
  \Begin{
    \nl $v(x) \leftarrow \sqrt{x}$ for $x\in\{1, \dots, |\constrDom|\}$ \\ 
    \nl $\alpha \leftarrow v(n-2)+v(n)-2v(n-1)$ \\
    \tcc{compute $\beta'$ of Eq.\eqref{eq:dssa:beta}}
    \ForEach{$x\in\{1, \dots, |\queryPool(\constrDom)|\}$}
    {
      \nl $r_1(x) \leftarrow \paren{1-\frac{1}{s} }^{\batchBudget} - \paren{1-\frac{1}{2s} }^{\batchBudget}$ \\
      \nl $r_2(x) \leftarrow \max_{\mathcal{T}: |\mathcal{T}| \leq s} \sum_{\ex\in \mathcal{T} } f(\ex)$ \label{ln:alg:dssa:s2}
    }
    \nl $\beta' \leftarrow -\max_x{r_1(x)r_2(x)}$ \label{ln:alg:dssa_lb}\\
    \nl $\dsf_1(\constrSelected) \leftarrow \frac{|\beta'|}{\alpha} v(|\constrSelected|)$ \\
    \nl $\dsg_1(\constrSelected) \leftarrow \batchUtilApprox(\constrSelected) + \frac{|\beta'|}{\alpha} v(|\constrSelected|)$ \\
    \nl {\bf Output}: DS decomposition $\batchUtilApprox=-(\dsf_1-\dsg_1)$\\
  }
  \caption{DS Construction via Submodular Augmentation (\algdssa)}\label{alg:lb-ds}
\end{algorithm}

\subsection{DS Construction via Submodular Augmentation}\label{sec:dssa}
It is well-established that every set function can be expressed as the sum of a submodular and a supermodular function \citep{narasimhan2005submodular}. In particular, \cite{iyer2012algorithms} provide the following constructive procedure for decomposing a set function into the DS form: Given a set function $q$, one can define $\beta = \min_{\constrSelected \subseteq \constrSelected' \subseteq \constrDom\setminus j} \Delta_q(j \given \constrSelected) - \Delta_q(j \given \constrSelected' )$, where $\Delta_q(j \given \constrSelected) := q(\constrSelected \cup \{j\}) - q(\constrSelected)$ denotes the gain of adding $j$ to $\constrSelected$. When $q$ is not submodular, we know that $\beta < 0$. Now consider any strictly submodular function $p$, with $\alpha = \min_{\constrSelected \subseteq \constrSelected' \subseteq \constrDom\setminus j} \Delta_p(j \given \constrSelected) - \Delta_p(j \given \constrSelected') > 0$. Define $h(\constrSelected) = q(\constrSelected) + \frac{|\beta'|}{\alpha} p(\constrSelected)$ for any $\beta' < \beta$. It is easy to verify that $h$ is submodular since $\min_{\constrSelected \subseteq \constrSelected' \subseteq \constrDom\setminus j} \Delta_h(j \given \constrSelected) - \Delta_h(j \given \constrSelected') \geq \beta + |\beta'| \geq 0$.
Hence $q(\constrSelected) = h(\constrSelected) - \frac{|\beta|}{\alpha} p(\constrSelected)$ is a difference between two submodular functions.

We refer to the above decomposition strategy as \algdssa (where \textsf{SA} stands for ``submodular augmentation''), and present the pseudo code in \algref{alg:lb-ds}. As suggested in
\citep{iyer2012algorithms}, we choose the submodular augmentation function $p(\constrSelected) = v(|\constrSelected|)$, where $v(x)$ is a concave function, and therefore $\alpha = \min_{x\leq x'; x,x' \subseteq \integers} v(x+1) - v(x) - v(x'+1) - v(x')$. This leads to the following decomposition 
of our surrogate objective $\batchUtilApprox$:
\begin{align}
  \label{eq:ds-sa}
  &\batchUtilApprox(\constrSelected)
  = \underbrace{\paren{\batchUtilApprox(\constrSelected) + \frac{|\beta'|}{\alpha} v(|\constrSelected|)}}_{\dsg_1(\constrSelected)} - \underbrace{\frac{|\beta'|}{\alpha} v(|\constrSelected|)}_{\dsf_1(\constrSelected)},
\end{align}
where $\dsf_1$, $\dsg_1$ by construction are submodular functions, and $\beta'$ is a lower bound on $\beta$:
\begin{align}\label{eq:dssa:beta}
    \beta' \leq \beta = \min_{\constrSelected \subseteq \constrSelected' \subseteq \constrDom\setminus j} \Delta_{\batchUtilApprox} (j \given \constrSelected) - \Delta_{\batchUtilApprox} (j \given \constrSelected' ).
\end{align}
The key step of the \algdssa algorithm is to construct such a lower bound $\beta'$. The following lemma, which is proved in the Appendix, shows that one can compute $\beta'$ in polynomial time, and hence can efficiently express $\batchUtilApprox$ as a DS as defined in Eq.~\eqref{eq:ds-sa}.
\begin{lemma}\label{lm:lb-ds}
  \algref{alg:lb-ds} 
  returns a DS-decomposition of $\batchUtilApprox$ in polynomial time.
\end{lemma}
\begin{algorithm}[t]
  \nl {\bf Input}: Constraints set $\prod_{\pid=1}^{\numpos} \constrDom_\pid$; budget on each batch $n$; selected constraints $\constrSelected$\\ 
  \Begin{
    \nl $u(x) \leftarrow \frac{x^2}{2}$, $\alpha \leftarrow 1$\\
    \nl $\mulfun(x) \leftarrow \paren{1-\frac{1}{x} }^{\batchBudget}$, \\
    \nl $\beta \leftarrow \abs{\min_x \mulfun''(x)}$ for $x\in \{1, \dots, |\queryPool(\constrDom)|\}$. \\
    \nl $\dsf_2(\constrSelected) \leftarrow - \paren{1+\frac{\beta}{\alpha} u(\queryPool(|\constrSelected|))} {\sum_{\ex \in \queryPool(\constrSelected)} {\utility(\ex)}}$ \\
    \nl $\dsg_2(\constrSelected) \leftarrow - \paren{\mulfun(|\queryPool(\constrSelected)|)  + \frac{\beta}{\alpha} u(\queryPool(|\constrSelected|))}  \sum_{\ex \in \queryPool(\constrSelected)} {\utility(\ex)} $ \\
     \nl {\bf Output}: DS decomposition $\batchUtilApprox=-(\dsf_2-\dsg_2)$\\
  }
  \caption{DS Construction via DC Decomposition (\algdsdc)}\label{alg:dc-ds}
\end{algorithm}
\subsection{DS Construction via DC Decomposition }\label{sec:dsdc}
We now consider an alternative strategy for decomposing the surrogate function $\batchUtilApprox$, based on a novel construction procedure that reduces to expressing a continuous function as the difference of convex (DC) functions. Concretely, we note that $\batchUtilApprox(\constrSelected)$ consists of two (multiplicative) terms: (i) a supermodular set function $\sum_{\ex \in \queryPool(\constrSelected)} \utility(\ex)$, and (ii) a set function that only depends on the cardinality of the input, i.e., $\sum_{\ex \in \queryPool(\constrSelected)} \utility(\ex) \paren{1-\paren{1-\frac{1}{|\queryPool(\constrSelected)|} }^{\batchBudget}}$.  As is further discussed in the Appendix, we show that one can exploit this structure, and focus on the DC decomposition of term (ii). We provide the detailed algorithm in \algref{alg:dc-ds}, and refer to it as \algdsdc (where \textsf{DC} stands for ``difference of convex decomposition'').

It is easy to check from \algref{alg:lb-ds} that \algdssa 
runs in quadratic time w.r.t. $|\queryPool(\constrDom)|$. 
In contrast, \algdsdc only requires finding the minimum of an array of size $|\queryPool(\constrDom)|$, which, in the best case, runs in linear time w.r.t. $|\queryPool(\constrDom)|$. 
At the end of the algorithm, \algdsdc outputs the following DS function:
\begin{align}
  \label{eq:dc-ds-construction}
  \batchUtilApprox(\constrSelected)
  &= \underbrace{\sum_{\ex \in \queryPool(\constrSelected)} \paren{-\rho(\ex)} \cdot \paren{\mulfun(|\queryPool(\constrSelected)|)  + \frac{\beta}{\alpha} u(\queryPool(|\constrSelected|))}}_{\dsg_2(\constrSelected)}
  \nonumber \\
  & \quad - \underbrace{{\sum_{\ex \in \queryPool(\constrSelected)} \paren{-\rho(\ex)}\paren{1+\frac{\beta}{\alpha} u(\queryPool(|\constrSelected|))}}}_{\dsf_2(\constrSelected)}.
\end{align}
where $u(x)$ is a non-negative, monotone convex function\footnote{In practice we often set $u(x) = \frac{x^2}{2}$ and thus $\alpha=1$.}, $\alpha = \min_x u''(x)$, $\mulfun(x) = \paren{1-\frac{1}{x} }^{\batchBudget}$, and $\beta =\abs{\min_x \mulfun''(x)}$. 
We then prove the following results:
\begin{lemma}\label{lm:dc-ds}
  With the decomposition as defined in Eq.~\eqref{eq:dc-ds-construction}, both functions $\dsf$, $\dsg$ are submodular and hence we obtain a DS-decomposition of $\batchUtilApprox$.
\end{lemma}






\section{Experiments}
In this section, we empirically evaluate our algorithm on three datasets. We first describe the datasets, then we provide empirical justification for our surrogate objective. Finally we provide quantitative evaluation of our algorithm against a few simple greedy heuristics.  

\subsection{Datasets} 
We test our algorithms on two experimental protein-engineering datasets and a synthetic dataset designed to have multiple local minima. The synthetic dataset has two sites with $\constrDomAt{1} = \constrDomAt{2} = 26$. Values for the items in the library are constructed such that there are disjoint blocks of items with non-zero $\rho(\ex)$ separated by regions where $\rho(x) = 0$. This guarantees that there are multiple local optima in the constraint space. The experimental datasets consist of measured fitness values for every sequence in four-site SSM libraries for protein G domain B1 (GB1) \citep{wu2016adaptation}, an immunoglobulin binding protein, and the protein kinase PhoQ \citep{podgornaia2015pervasive}. These fitness landscapes are known to have high levels of multi-site epistasis. Having measurements for every fitness value in each library allows us to simulate engineering via multiple rounds of SSM. 


\subsection{Suitability of the surrogate objective}
\algname uses a GP posterior to model the unobserved utilities. However, there is no closed form for the true batch constraint design objective, which is to choose constraints $\constrSelected$ that maximize the expected number of improved observations found by querying $\queryPool(\constrSelected)$. The approximate objective function (Equation~\ref{eq:batch-util-approx}) ignores dependencies between items, and thus will overestimate the true objective. To test the suitability of the approximate objective function, we selected an initial batch of sequences from the PhoQ dataset consisting of all the single mutants plus 100 randomly-selected sequences, trained a GP regression model, and used the posterior to compute the rewards $\rho(x)$. Figure~\ref{fig:exp:comparison} shows that values for the approximate objective are well-correlated with the true objective estimated using Monte Carlo sampling. However, the independence assumption leads to overestimating the number of improved sequences that will be found. 

\begin{figure}[t]
	\centering
	\includegraphics[trim={0pt 0pt 0pt 0pt}, width=.35\textwidth]{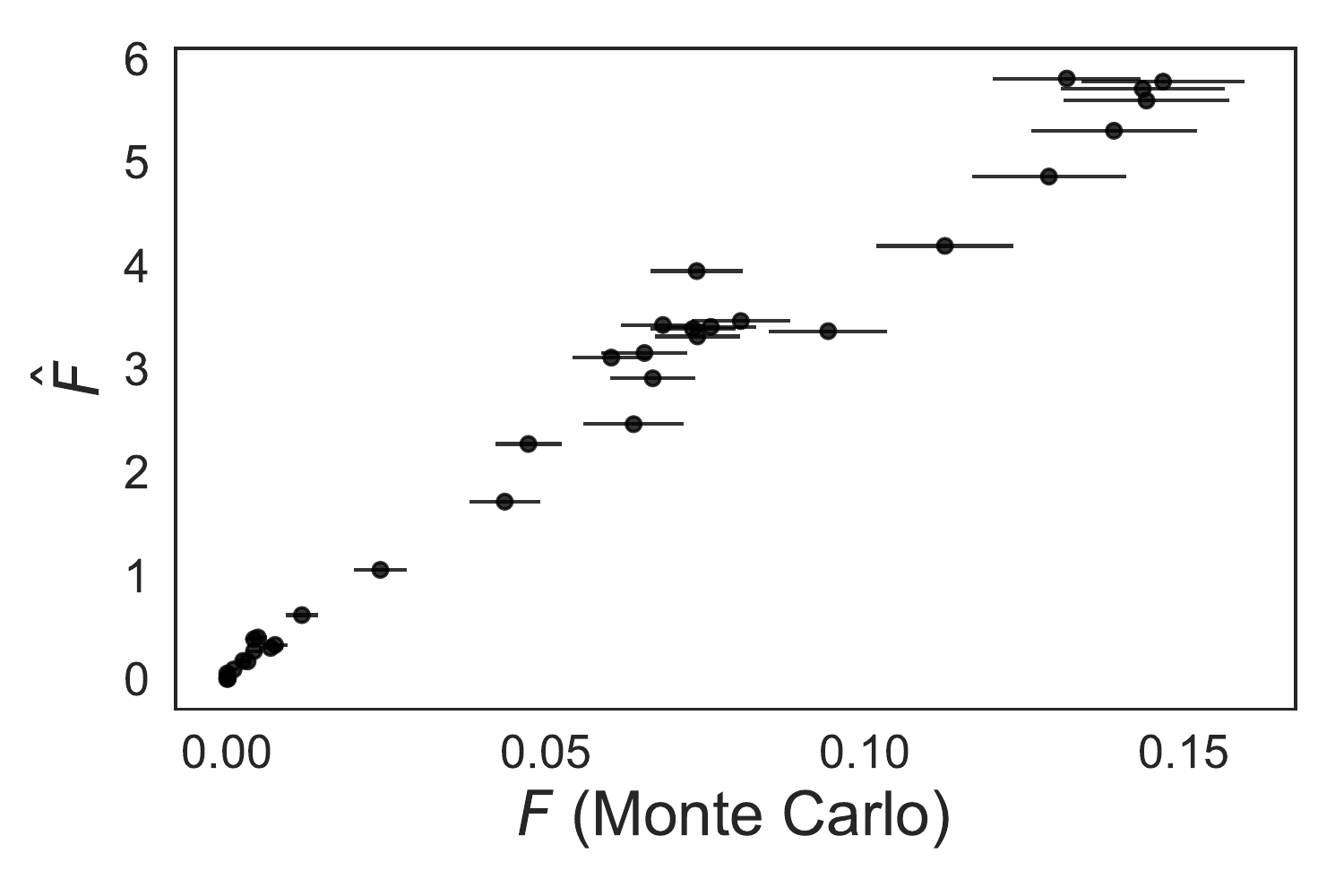}
	\caption{Comparing the full objective function approximated using Monte Carlo sampling and approximated with an independence assumption, as in Equation~\ref{eq:batch-util-approx}. Error bars are standard errors for the Monte Carlo estimates. }
	\label{fig:exp:comparison}
\end{figure}

\begin{figure}[t]
	\centering
	\includegraphics[trim={0pt 0pt 0pt 0pt}, width=.35\textwidth]{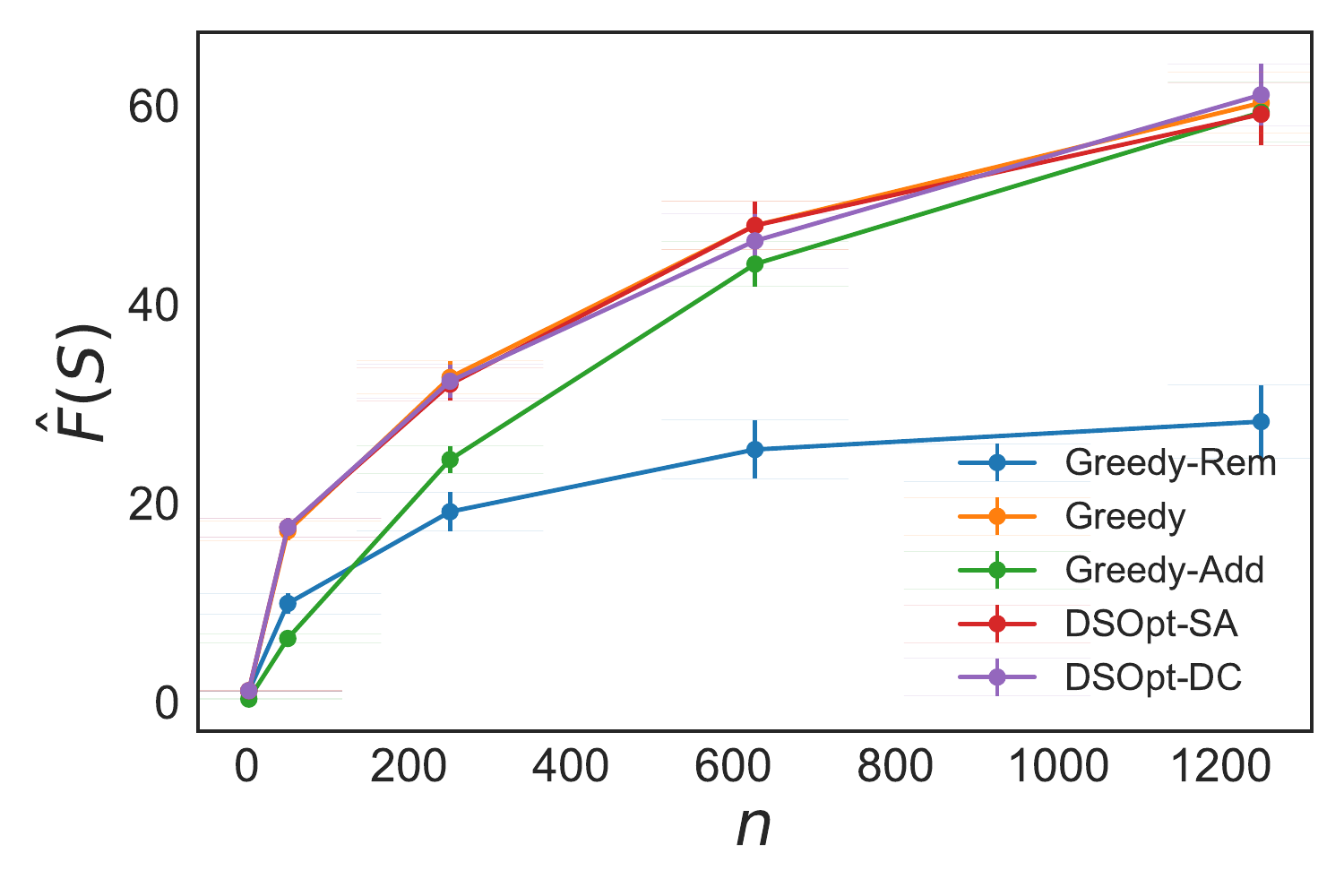}
	\caption{Performance of each algorithm on finding constraints for the synthetic dataset.  Error bars are standard errors. }
	\label{fig:exp:alg}
	\vspace{-3mm}
\end{figure}

\begin{figure*}[t]
  \centering
  \begin{subfigure}[b]{.24\textwidth}
    \centering
    {
      \includegraphics[trim={0pt 0pt 0pt 0pt}, width=\textwidth]{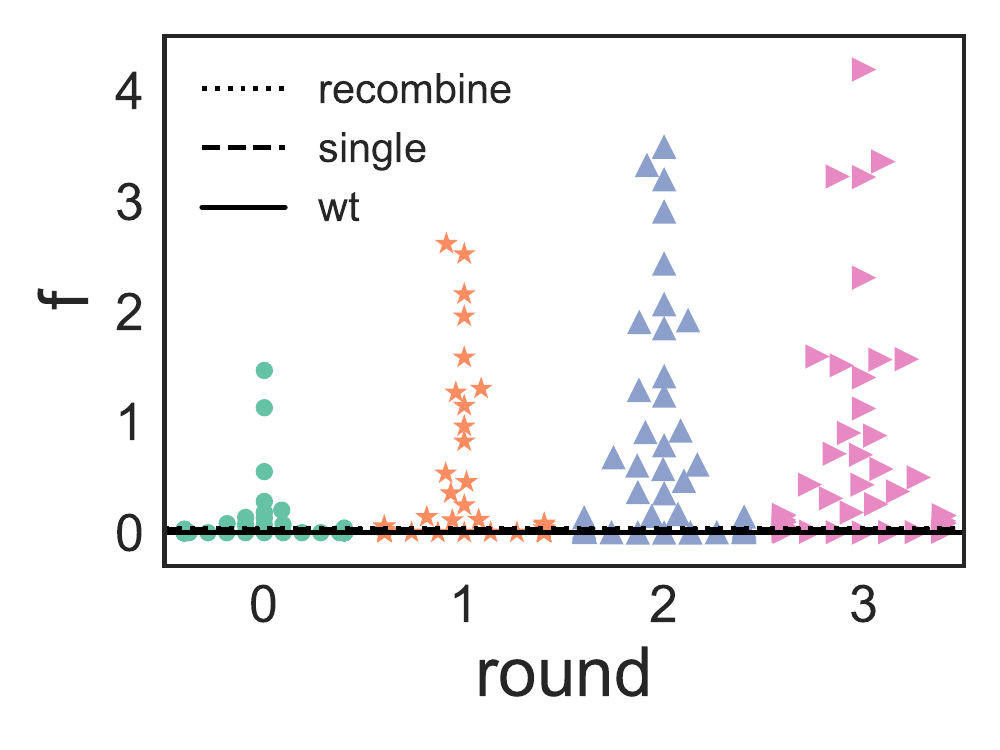}
      \caption{Simulation for GB1}
      \label{fig:gb1-batches}
    }
  \end{subfigure}
  \begin{subfigure}[b]{.24\textwidth}
    \centering
    {
      \includegraphics[trim={0pt 0pt 0pt 0pt}, width=\textwidth]{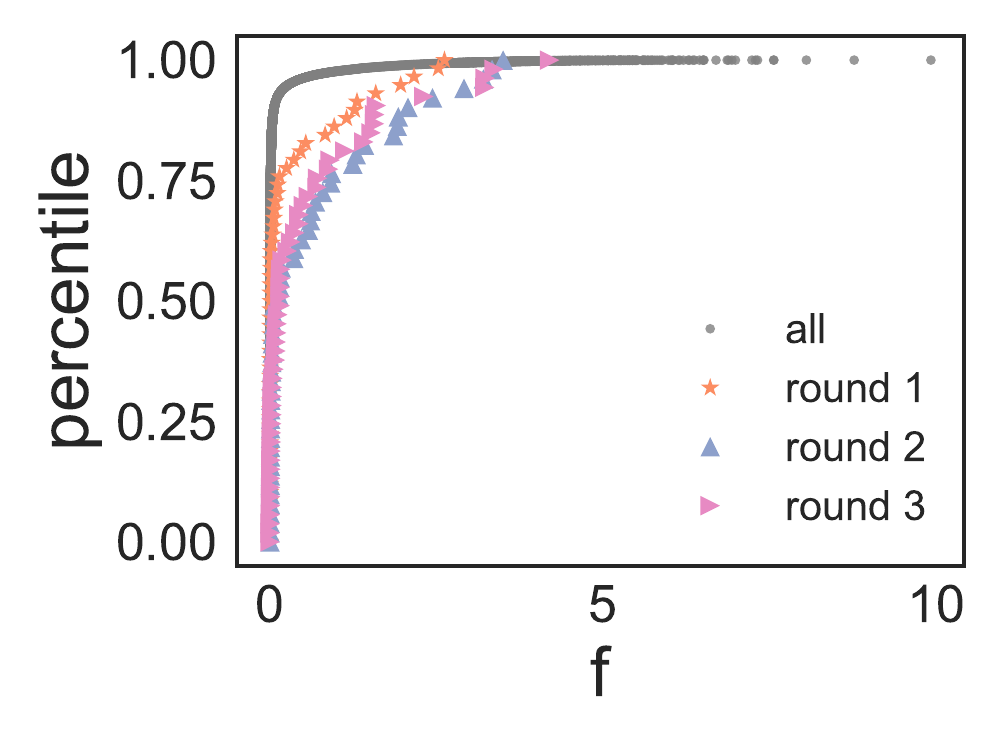}
      \caption{e-CDFs for GB1}
      \label{fig:gb1-ecdf}
    }
  \end{subfigure}
  \begin{subfigure}[b]{.24\textwidth}
    \centering
    {
      \includegraphics[trim={0pt 0pt 0pt 0pt}, width=\textwidth]{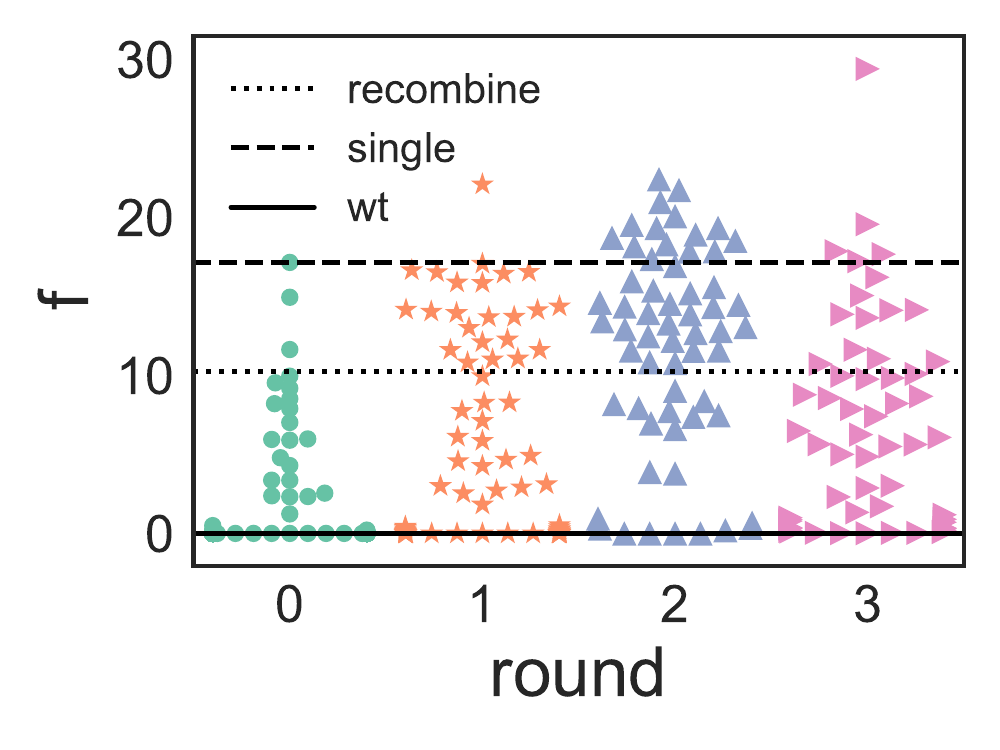}
      \caption{Simulation for PhoQ}
      \label{fig:phoq-batches}
    }
  \end{subfigure}
  \begin{subfigure}[b]{.24\textwidth}
    \centering
    {
      \includegraphics[trim={0pt 0pt 0pt 0pt}, width=\textwidth]{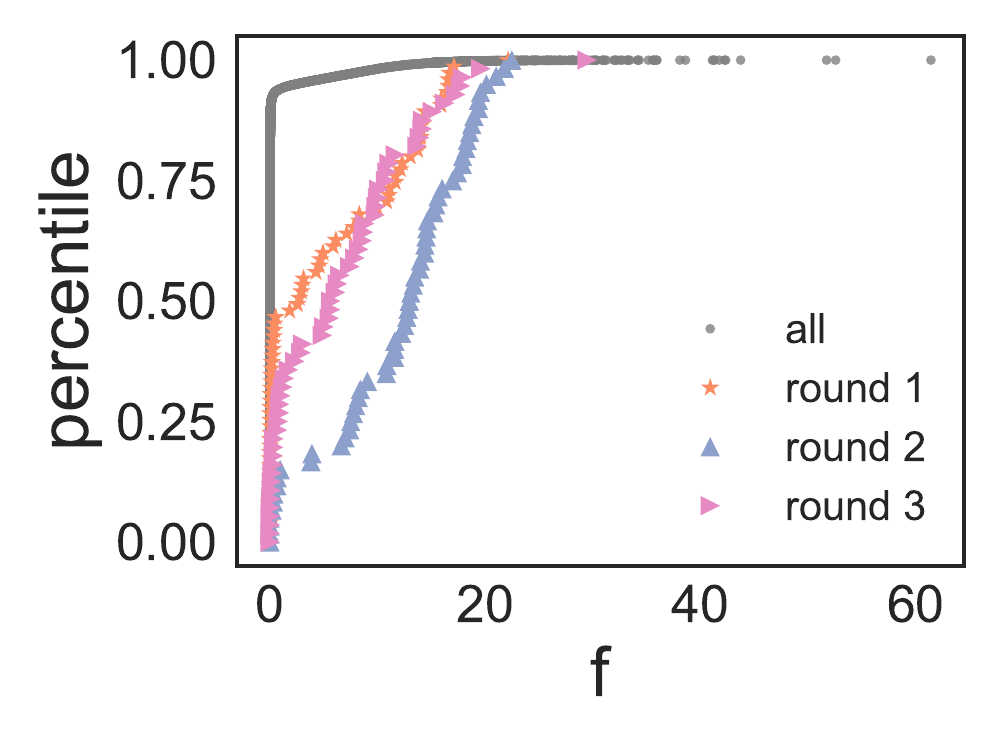}
      \caption{e-CDFs for PhoQ}
      \label{fig:phoq-ecdf}
    }
  \end{subfigure}
  \caption{Experimental results. (a) and (c) show fitness values for the sequences sampled at each round for GB1 and PhoQ, respectively. The solid horizontal lines show the fitnesses for the wild-type sequences (wt). The dashed horizontal lines show the fitnesses for the best sequences with exactly one mutation from the wild type (single). The dotted horizontal lines show the fitnessees for the sequences that combine the best amino acid at each site determined in the wild-type background. (b) and (d) show empirical cumulative distribution functions (e-CDFs) for the entire library and for each sequence selected using \algname in rounds 1 - 3 for GB1 and PhoQ, respectively.}
  \label{fig:exp:real}
\end{figure*}

\subsection{Algorithm comparisons}

Next, we compare the performance of \algname using \alglbds and \algdcds against three greedy variants using the synthetic dataset. \algadd greedily adds constraints, \algrem greedily removes constraints, and \alggreedy greedily adds or removes constraints until the objective stops improving. Because the empty set is a local optimum (adding any single constraint still results in no valid queries), it is necessary to begin the optimization at a set of constraints that yields a non-empty set of queries. 

We compare the algorithms at a range of batchsizes $n$. At each $n$, we initialize each algorithm at $\constrDom$, the constraints that result in the single best query, and 18 randomly selected sets of constraints. Figure~\ref{fig:exp:alg} shows that \alglbds, \algdcds, and \alggreedy strongly outperform \algadd and \algrem across all values of $n$. At small $n$, the optima tend to have few constraints, so \algadd performs particularly poorly. As $n$ approaches infinity, the optimum approaches the ground set $\constrDom$, and so \algrem performs particularly poorly. \alglbds, \algdcds, and \alggreedy perform very similarly across all values of $n$. In theory, \alglbds and \algdcds can escape local optima to find better solutions than \alggreedy, but this appears to be rare on this dataset. There is also no guarantee that \alglbds or \algdcds will converge to a better optimum instead of merely a different optimum. 

\subsection{Simulation on protein datasets}

We use the PhoQ and GB1 datasets to simulate the ability of \algname to select constraints that result in libraries enriched in improved sequences. For both PhoQ and GB1, we initiated the simulation by selecting an initial batch of sequences consisting of all the single mutants plus 100 randomly-selected sequences. We then ran three iterations of the algorithm with batchsize $\batchBudget=100$, resulting in three more batched queries. This simulates an SSM experiment with 3 rounds of diversification, screening, and selection. The batchsize was chosen to approximate the number of samples that fit on a 96-well plate. At each iteration, we train a GP regression model using a Mat\'ern kernel with $\nu = \frac{2}{2}$ in order to compute the rewards $\rho(\ex)$. 

For both GB1 and PhoQ, \algname finds improved sequences. Importantly, it finds much better sequences than combining the best mutation at each site in the wild-type background or the best sequence with a single mutation from the wild-type, as shown in Figure~\ref{fig:gb1-batches} and Figure~\ref{fig:phoq-batches}. These are common experimental heuristics for dealing with multi-site SSM libraries where the library is too large to reasonably screen. Figure~\ref{fig:gb1-ecdf} and Figure~\ref{fig:phoq-ecdf} show that \algname finds extremely-rare ($> 99.8^{\text{th}}$ percentile) sequences that would be extremely difficult to find by randomly sampling $< 500$ sequences from the entire library. In epistatic landscapes such as for GB1 and PhoQ, considering multiple sites simultaneously is necessary to escape local optima in the sequence-function landscape. In GB1, only looking at single mutations or recombining the best mutations at each site results in very poor fitnesses with no improvement over the wild-type. In PhoQ, the best single mutant has a higher fitness than recombining the best mutations at each site, demonstrating the importance of epistasis. \algname reduces the library size for a multi-site SSM library in order to increase the probability of finding sequences with improved fitness values.


\section{Conclusion}

In this paper, we investigated a novel Bayesian optimization problem: batched stochastic Bayesian optimization. This problem setting poses two unique challenges: optimizing over the space of constraints instead of directly over items and stochastic sampling. We proposed an effective online optimization framework for searching through the combinatorial design space of constraints in order to maximize the expected number of improved items sampled at each iteration. In particular, we proposed a novel approximate objective function that links a model trained on the individual items to the constraint space and derived two efficient DS decompositions for this objective. Our method efficiently finds sequences with improved fitnesses in fully-characterized SSM libraries for the proteins GB1 and PhoQ, demonstrating its potential to enable engineering via simultaneous SSM even in cases where it is not feasible to measure more than a tiny fraction of the sequences in the library. 

\section*{Acknowledgments} 
This work was supported in part by the Donna and Benjamin M. Rosen Bioengineering Center, the U.S. Army Research Office Institute for Collaborative Biotechnologies, NSF Award \#1645832, Northrop Grumman, Bloomberg, PIMCO, and a Swiss NSF Early Mobility Postdoctoral Fellowship.



\bibliographystyle{icml2018}
\bibliography{reference}

\iftoggle{longversion}{
\appendix
\onecolumn

\section{Proofs}

\begin{lemma}\label{lm:sum-supermodular}
  Let $g(\constrSelected)=\sum_{\ex \in \queryPool(\constrSelected)} \utility(\ex)$, where $\queryPool(\constrSelected)$ is defined in Eq. \eqref{eq:querypool}. If $\forall \ex,\ \utility(\ex) \geq 0$, then $g$ is monotone supermodular.
\end{lemma}
\begin{proof}
  Let $\pid \in [\numpos]$, and $j\in \constrDomAt{\pid}$ be any constraint at site $\pid$. For $\constrSelected \subseteq \constrDom \setminus \{j\}$, define $\Delta_g(j \given \constrSelected ) = \sum_{\ex \in \queryPool(\constrSelected \cup \{j\})} \utility(\ex) - \sum_{\ex \in \queryPool(\constrSelected)} \utility(\ex) $ to be the gain of adding $j$ to the set $\constrSelected$.

  By definition of $\queryPool(\constrSelected)$, we have $\queryPool(\constrSelected) = {\prod_{\varpid=1}^{\numpos} \constrSelectedAt{\varpid}}$, and
  \begin{align}
    \queryPool(\constrSelected \cup \{j\})
    &= \paren{\constrSelectedAt{\pid} \cup \{j\}} \times \prod_{\varpid\neq \pid} \constrSelectedAt{\varpid} \nonumber \\
    &= \paren{{\{j\}} \times \prod_{\varpid\neq \pid} \constrSelectedAt{\varpid}} \bigcup  \paren{\constrSelectedAt{\pid} \times \prod_{\varpid\neq \pid} \constrSelectedAt{\varpid}} \nonumber \\
    &= \paren{{\{j\}} \times \prod_{\varpid\neq \pid} \constrSelectedAt{\varpid}} \bigcup  \paren{\prod_{\varpid=1}^{\numpos} \constrSelectedAt{\varpid}} \label{eq:querypool-plusj}
  \end{align}
  Then,
  \begin{align*}
    \Delta_g(j \given \constrSelected )
    = \sum_{\ex \in \queryPool(\constrSelected \cup \{j\})} \utility(\ex) - \sum_{\ex \in \queryPool(\constrSelected)} \utility(\ex)
    \stackrel{Eq.~\eqref{eq:querypool-plusj}}{=} \sum_{\ex \in {\{j\}} \times \prod_{\varpid\neq \pid} \constrSelectedAt{\varpid}}  \utility(\ex)
  \end{align*}
  Now let us consider $\constrSelected'$ such that $\constrSelected \subseteq \constrSelected' \subseteq \constrDom \setminus \{j\}$. Clearly $\forall \varpid\in [\numpos],\ \constrSelected^{\paren\varpid} \subseteq \constrSelected'^{\paren\varpid}$. 
  Therefore, $\Delta_g(j \given \constrSelected' ) - \Delta_g(j \given \constrSelected )
  = \sum_{\ex \in {\{j\}} \times \prod_{\varpid\neq \pid} \paren{\constrSelected'^{\paren\varpid} \setminus \constrSelectedAt{\varpid}}}  \utility(\ex)
  \geq 0$
  and hence $g$ is supermodular.
\end{proof}


\subsection{Proof of \lemref{lm:lb-ds}}
We now show that \algref{alg:lb-ds} leads to a polynomial algorithm for constructing a lower bound on Eq.~\eqref{eq:dssa:beta}, and hence on constructing a DS-decomposition of the surrogate objective function $\batchUtilApprox$ (Eq.~\eqref{eq:batch-util-approx}).

\begin{proof}[Proof of \lemref{lm:lb-ds}]
  Let $g(\constrSelected)=\sum_{\ex \in \queryPool(\constrSelected)} \utility(\ex)$. By definition we have
  \begin{align*}
    \batchUtilApprox(\constrSelected)
    = g(\constrSelected) \paren{1-\paren{1-\frac{1}{|\queryPool(\constrSelected)|} }^{\batchBudget} }
    = \underbrace{g(\constrSelected)}_{\batchUtilApprox_1(\constrSelected)} - \underbrace{g(\constrSelected){\paren{1-\frac{1}{|\queryPool(\constrSelected)|} }^{\batchBudget} }}_{\batchUtilApprox_2(\constrSelected)}
    = \batchUtilApprox_1(\constrSelected) - \batchUtilApprox_2(\constrSelected)
  \end{align*}

 We know from \lemref{lm:sum-supermodular} that $\batchUtilApprox_1$ is supermodular. Let $j \in \constrDom$ and $\constrSelected \subseteq \constrDom \setminus \{j\}$. The gain of $j$ on $\batchUtilApprox_1$, denote by $\Delta_1(j\given \constrSelected)$, is monotone decreasing. 
 
 Let $\Delta_2(j\given \constrSelected)
    = \batchUtilApprox_2(\constrSelected \cup \{j\}) - \batchUtilApprox_2(\constrSelected)$. Our goal is to find a lower bound on 
 \begin{align}
 \beta &= \min_{\constrSelected \subseteq \constrSelected' \subseteq \constrDom\setminus j} \paren{\Delta_{\batchUtilApprox} (j \given \constrSelected) - \Delta_{\batchUtilApprox} (j \given \constrSelected')} \nonumber \\
 &=  \min_{\constrSelected \subseteq \constrSelected' \subseteq \constrDom\setminus j} \paren{\underbrace{\Delta_{1} (j \given \constrSelected) - \Delta_{1} (j \given \constrSelected')}_{\geq 0} + \Delta_{2} (j \given \constrSelected) - \Delta_{2} (j \given \constrSelected')}
 \end{align}
 Therefore, it suffices to find a lower bound $\Delta_{2} (j \given \constrSelected) - \Delta_{2} (j \given \constrSelected')$.
The gain of $j$ on $\batchUtilApprox_2$ is 
  \begin{align*}
    \Delta_2(j\given \constrSelected)
    &= \batchUtilApprox_2(\constrSelected \cup \{j\}) - \batchUtilApprox_2(\constrSelected) \\
      &= \sum_{\ex \in \queryPool(\constrSelected \cup \{j\})} f(\ex) {\paren{1-\frac{1}{|\queryPool(\constrSelected \cup \{j\})|} }^{\batchBudget} } - \sum_{\ex \in \queryPool(\constrSelected)} f(\ex) {\paren{1-\frac{1}{|\queryPool(\constrSelected)|} }^{\batchBudget} } \\
    &= \sum_{\ex \in \queryPool(\constrSelected \cup \{j\}) \setminus \queryPool(\constrSelected)} f(\ex) {\paren{1-\frac{1}{|\queryPool(\constrSelected \cup \{j\})|} }^{\batchBudget} } + \\
    &\qquad \qquad \sum_{\ex \in \queryPool(\constrSelected)} f(\ex) \paren{{\paren{1-\frac{1}{|\queryPool(\constrSelected \cup \{j\})|} }^{\batchBudget} }-{\paren{1-\frac{1}{|\queryPool(\constrSelected)|} }^{\batchBudget} }}
  \end{align*}
  Let $r(\constrSelected) = \paren{1-\frac{1}{|\queryPool(\constrSelected)|} }^{\batchBudget}$. Then, the above equation can be simplified as
  \begin{align*}
    \Delta_2(j\given \constrSelected)
    &= \batchUtilApprox_2(\constrSelected \cup \{j\}) - \batchUtilApprox_2(\constrSelected) \\
      &= \underbrace{\sum_{\ex \in \queryPool(\constrSelected \cup \{j\}) \setminus \queryPool(\constrSelected)} f(\ex) r(\constrSelected \cup \{j\})}_{T_1(\constrSelected)} + \underbrace{\sum_{\ex \in \queryPool(\constrSelected)} f(\ex) \paren{r(\constrSelected \cup \{j\})-r(\constrSelected)}}_{T_2(\constrSelected)}
  \end{align*}
  It is easy to verify that $T_1(\constrSelected)$ is monotone increasing function of $\constrSelected$.   Let us consider $\constrSelected'$ such that $\constrSelected \subseteq \constrSelected' \subseteq \constrDom \setminus \{j\}$. We have
  \begin{align}
    \Delta_2(j\given \constrSelected') - \Delta_2(j\given \constrSelected)
    &\stackrel{}{\geq} T_2(\constrSelected') - T_2(\constrSelected) \nonumber \\
    &\stackrel{T_2\geq 0}{\geq} -g(\constrSelected)(r(\constrSelected \cup \{j\}) - r(\constrSelected))\nonumber
  \end{align}
  Therefore, it suffices to find a lower bound on $-g(\constrSelected)(r(\constrSelected \cup \{j\}) - r(\constrSelected))$. Further notice that
  \begin{align}
    \label{eq:app:sa:lb:term1}
    0 \leq g(\constrSelected) \leq \max_{\mathcal{T}: |\mathcal{T}| \leq |\queryPool(\constrSelected)|} \sum_{\ex\in \mathcal{T} } f(\ex)
  \end{align}
  and it is not hard to verify that
    \begin{align}
    \label{eq:app:sa:lb:term2}
    0\leq r(\constrSelected \cup \{j\}) - r(\constrSelected) \leq \paren{1-\frac{1}{|\queryPool(\constrSelected)|} }^{\batchBudget} - \paren{1-\frac{1}{2|\queryPool(\constrSelected)|} }^{\batchBudget}
  \end{align}
Therefore, combining term \eqref{eq:app:sa:lb:term1} with \eqref{eq:app:sa:lb:term2}, we get a lower bound on $\beta$:
\begin{align}
    \beta \geq - \max_{s\in \{1, \dots, |\queryPool(\constrDom)|\}} \paren{\paren{\paren{1-\frac{1}{s} }^{\batchBudget} - \paren{1-\frac{1}{2s} }^{\batchBudget}} \underbrace{\max_{\mathcal{T}: |\mathcal{T}| \leq s} \sum_{\ex\in \mathcal{T} } f(\ex)}_{\text{Term 2}}}\label{eq:app:beta:lowerbound}
\end{align}
Note that term 2 is a modular function and can be optimized greedily. Therefore, computing the RHS of Eq.~\ref{eq:app:beta:lowerbound} can be efficiently done in polynomial time w.r.t. $|\queryPool(\constrDom)|$.
\end{proof}


\subsection{Proof of \lemref{lm:dc-ds}: Difference of Convex Construction of DS Decomposition}

\begin{lemma}\label{lm:supconvprod}
  Let $g: 2^\constrDom \rightarrow \NonNegativeReals$ be a non-negative, non-decreasing supermodular function, and $\mulconvfun: \reals \rightarrow \reals$ be a non-decreasing convex function. For $\constrSelected \subseteq \constrDom$, define $h(\constrSelected) = g(\constrSelected) \cdot \mulconvfun(|\constrSelected|)$. Then $h$ is supermodular.
\end{lemma}
\begin{proof}
  Let $j \in \constrDom$ and $\constrSelected \subseteq \constrDom \setminus \{j\}$. The gain of $j$ is
  \begin{align*}
    \Delta_h(j\given \constrSelected)
    &= h(\constrSelected \cup \{j\}) - h(\constrSelected) \\
    &= g(\constrSelected \cup \{j\}) \cdot \mulconvfun(|\constrSelected \cup \{j\}|) - g(\constrSelected) \cdot \mulconvfun(|\constrSelected|) \\
    &= \paren{g(\constrSelected \cup \{j\}) - g(\constrSelected)} \cdot \mulconvfun(|\constrSelected \cup \{j\}|) + g(\constrSelected) \paren{\mulconvfun(|\constrSelected \cup \{j\}|) -  \mulconvfun(|\constrSelected|)}
  \end{align*}
  Let us consider $\constrSelected'$ such that $\constrSelected \subseteq \constrSelected' \subseteq \constrDom \setminus \{j\}$. We have
  \begin{align*}
    \Delta_h(j\given \constrSelected)
    &= \paren{g(\constrSelected \cup \{j\}) - g(\constrSelected)} \cdot \mulconvfun(|\constrSelected \cup \{j\}|) + g(\constrSelected) \paren{\mulconvfun(|\constrSelected \cup \{j\}|) -  \mulconvfun(|\constrSelected|)}\\
    &\stackrel{(a)}{\leq} \paren{g(\constrSelected' \cup \{j\}) - g(\constrSelected')} \cdot \mulconvfun(|\constrSelected' \cup \{j\}|) + g(\constrSelected) \paren{\mulconvfun(|\constrSelected \cup \{j\}|) -  \mulconvfun(|\constrSelected|)}\\
    &\stackrel{(b)}{\leq} \paren{g(\constrSelected' \cup \{j\}) - g(\constrSelected')} \cdot \mulconvfun(|\constrSelected' \cup \{j\}|) + g(\constrSelected') \paren{\mulconvfun(|\constrSelected' \cup \{j\}|) -  \mulconvfun(|\constrSelected'|)}\\
    &= \Delta_h(j\given \constrSelected')
  \end{align*}
  where step (a) is due to  $g$ being monotone supermodular (i.e., $g(\constrSelected' \cup \{j\}) - g(\constrSelected') \geq g(\constrSelected \cup \{j\}) - g(\constrSelected) \geq 0$) and $u$ being monotone (i.e., $\mulconvfun(|\constrSelected' \cup \{j\}|) \geq \mulconvfun(|\constrSelected \cup \{j\}|)$); step (b) is due to $g$ being non-negative monotone (i.e., $g(\constrSelected') \geq g(\constrSelected) \geq 0$) and $u$ being convex (i.e., $\mulconvfun(|\constrSelected' \cup \{j\}|) -  \mulconvfun(|\constrSelected'|) \geq \mulconvfun(|\constrSelected \cup \{j\}|) -  \mulconvfun(|\constrSelected|)$). Therefore $h$ is supermodular.
\end{proof}
\begin{lemma}\label{lem:convex-composition}
  Let $w: \reals \rightarrow \reals$ be a convex function and $u: \reals \rightarrow \reals$ a convex non-decreasing function, then $u \circ w$ is convex. Furthermore, if $w$ is non-decreasing, then the composition is also non-decreasing.
\end{lemma}
\begin{proof}
  By convexity of $w$:
  \begin{align*}
    w(\alpha x + (1-\alpha) y) \leq \alpha w(x) + (1-\alpha) w(y).
  \end{align*}
  Therefore, we get
  \begin{align*}
    u(w(\alpha x + (1-\alpha) y))
    &\stackrel{(a)}{\leq} u\paren{\alpha w(x) + (1-\alpha)w(y)} \\
    &\stackrel{(b)}{\leq }\alpha u(w(x)) + (1-\alpha) u(w(y)).
  \end{align*}
  Here, step (a) is due to the fact that $u$ is non-decreasing, and step (b) is due to the convexity of $u$. Therefore $u\circ w$ is convex. If $w$ is non-decreasing, it is clear that $u\circ w$ is also non-decreasing, hence completes the proof.
\end{proof}

\begin{lemma}[\citet{horst1999dc}]\label{lem:dc-construction}
  Let $\mulfun:\reals\rightarrow \reals$ be a non-decreasing, twice continuously differentiable function. Then $\mulfun$ can be represented as the difference between two non-decreasing convex functions.
\end{lemma}
\begin{proof}
  Let $u: \reals \rightarrow \reals$ be a non-decreasing, strictly convex function, and $\alpha = \min_x u''(x)$; clearly, $\alpha > 0$.

  Let $\beta =\abs{\min_x \mulfun''(x)}$. Define
  \begin{align}\label{eq:dc-construction}
    v(x) = \mulfun(x) + \frac{\beta}{\alpha} u(x)
  \end{align}
  It is easy to verify that
  \begin{align*}
    v''(x)
    = \mulfun''(x) + \frac{\beta}{\alpha} u''(x)
    \geq \mulfun''(x) + \beta
    \geq 0.
  \end{align*}
  Hence, $v(x)$ is convex. Furthermore, since both $\mulfun$ and $u$ are non-decreasing, $v$ is also non-decreasing. Therefore, $\mulfun(x) = v(x) - \frac{\beta}{\alpha} u(x)$ is the difference between two non-decreasing convex functions.
\end{proof}
\begin{lemma}\label{lem:dc-compfunc}
  Let $\mulfun:\reals\rightarrow \reals$ be a non-decreasing, twice continuously differentiable function, and $w: \reals \rightarrow \reals$ a convex non-decreasing function, then $\mulfun \circ w$ can be represented as the difference between two non-decreasing convex functions.
\end{lemma}
\begin{proof}
  By \lemref{lem:dc-construction}, we can represent $\mulfun(x) = v(x) - \frac{\beta}{\alpha} u(x)$, where $u, v$ are non-decreasing convex functions, and $\alpha,\beta$ are as defined in Eq.~\eqref{eq:dc-construction}. Therefore,
  \begin{align*}
    \mulfun\circ w(x) = v \circ w(x) - \frac{\beta}{\alpha} \cdot u\circ w(x)
  \end{align*}
  By \lemref{lem:convex-composition}, $v \circ w$ and $u \circ w$ are both non-decreasing convex, which completes the proof.
\end{proof}


Now we are ready to prove \lemref{lm:dc-ds}.
\begin{proof}[Proof of \lemref{lm:dc-ds}]
  Let $g(\constrSelected)=\sum_{\ex \in \queryPool(\constrSelected)} \utility(\ex)$. By definition we have
  \begin{align*}
    \batchUtilApprox(\constrSelected)
    = g(\constrSelected) \paren{1-\paren{1-\frac{1}{|\queryPool(\constrSelected)|} }^{\batchBudget} }
    = g(\constrSelected) - g(\constrSelected){\paren{1-\frac{1}{|\queryPool(\constrSelected)|} }^{\batchBudget} }
  \end{align*}
  Let $\mulfun(x) = \paren{1-\frac{1}{x} }^{\batchBudget}$, and $w:\reals \rightarrow \reals$ be a convex function, such that $w(|\constrSelected|) = |\queryPool(\constrSelected)|$. Note that such function $w$ exists, because the set function $h(\constrSelected) := |\queryPool(\constrSelected)|$ is supermodular. Therefore, we have
  \begin{align*}
    \batchUtilApprox(\constrSelected)
    = g(\constrSelected) - g(\constrSelected) \cdot r\circ w(|\constrSelected|)
  \end{align*}
  Furthermore, note that $\mulfun$ is non-decreasing, twice continuously differentiable at $[1, \infty)$. By~\lemref{lem:dc-compfunc}, we get
  \begin{align}
    \batchUtilApprox(\constrSelected)
    &= g(\constrSelected) - g(\constrSelected) \cdot \paren{v \circ w(|\constrSelected|) - \frac{\beta}{\alpha} \cdot u\circ w(|\constrSelected|)} \nonumber \\
    &= g(\constrSelected)\paren{1+\frac{\beta}{\alpha} \cdot u\circ w(|\constrSelected|)} - g(\constrSelected) \cdot \paren{v \circ w(|\constrSelected|)}, \label{eq:app:dcds-construction}
  \end{align}
  where $u: \reals \rightarrow \reals$ can be any non-decreasing, strictly convex function, $\alpha = \min_x u''(x)$, $\beta =\abs{\min_{x\geq 1} \mulfun''(x)}$, and $v(x) = \mulfun(x) + \frac{\beta}{\alpha} u(x)$.

  We know from \lemref{lm:sum-supermodular} that $g$ is supermodular. Since both $1+\frac{\beta}{\alpha} \cdot u\circ w(x)$ and $v \circ w(x)$ are convex, then by \lemref{lm:supconvprod}, we know that both terms on the R.H.S. of Eq.~\eqref{eq:app:dcds-construction} are supermodular, and hence we obtain a DS decomposition of function $\batchUtilApprox$.
\end{proof}



\onecolumn
\newpage

\section{Supplemental Figures}

\renewcommand{\thefigure}{S\arabic{figure}}

\setcounter{figure}{0}

\begin{figure}[!h]
    \centering
	\includegraphics[trim={0pt 0pt 0pt 0pt}, width=.5\textwidth]{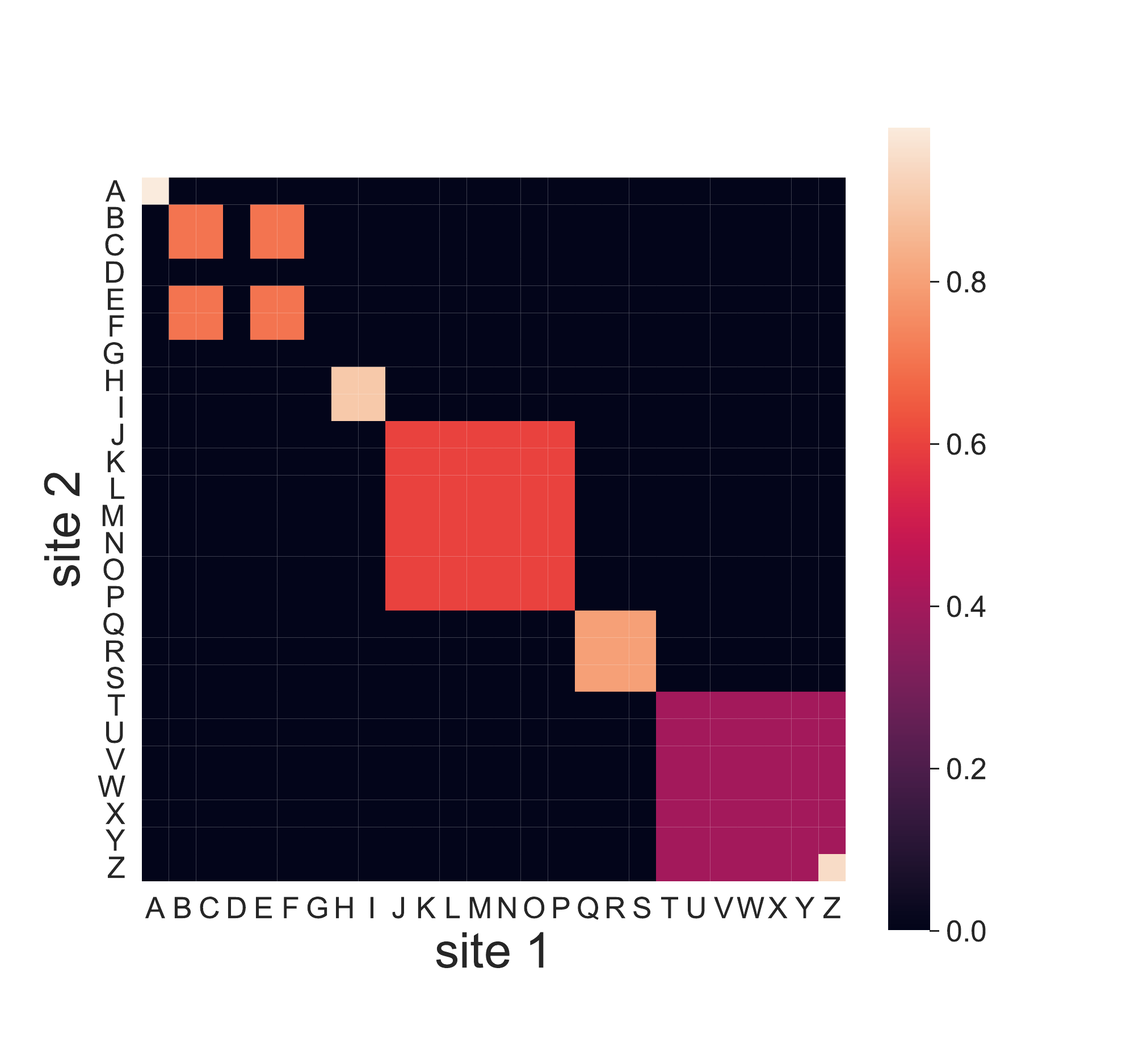}
	\caption{The cell values for the synthetic dataset with  $\numpos = 2$ and $|\constrDomAt{\pid}| = 26 \ \forall \pid\in \{1, 2\}$.}
	\label{fig:exp:heatmap}
\end{figure}

\begin{figure}[!h]
  \centering
  \begin{subfigure}[b]{.35\textwidth}
    \centering
    {
      \includegraphics[trim={0pt 0pt 0pt 0pt}, width=\textwidth]{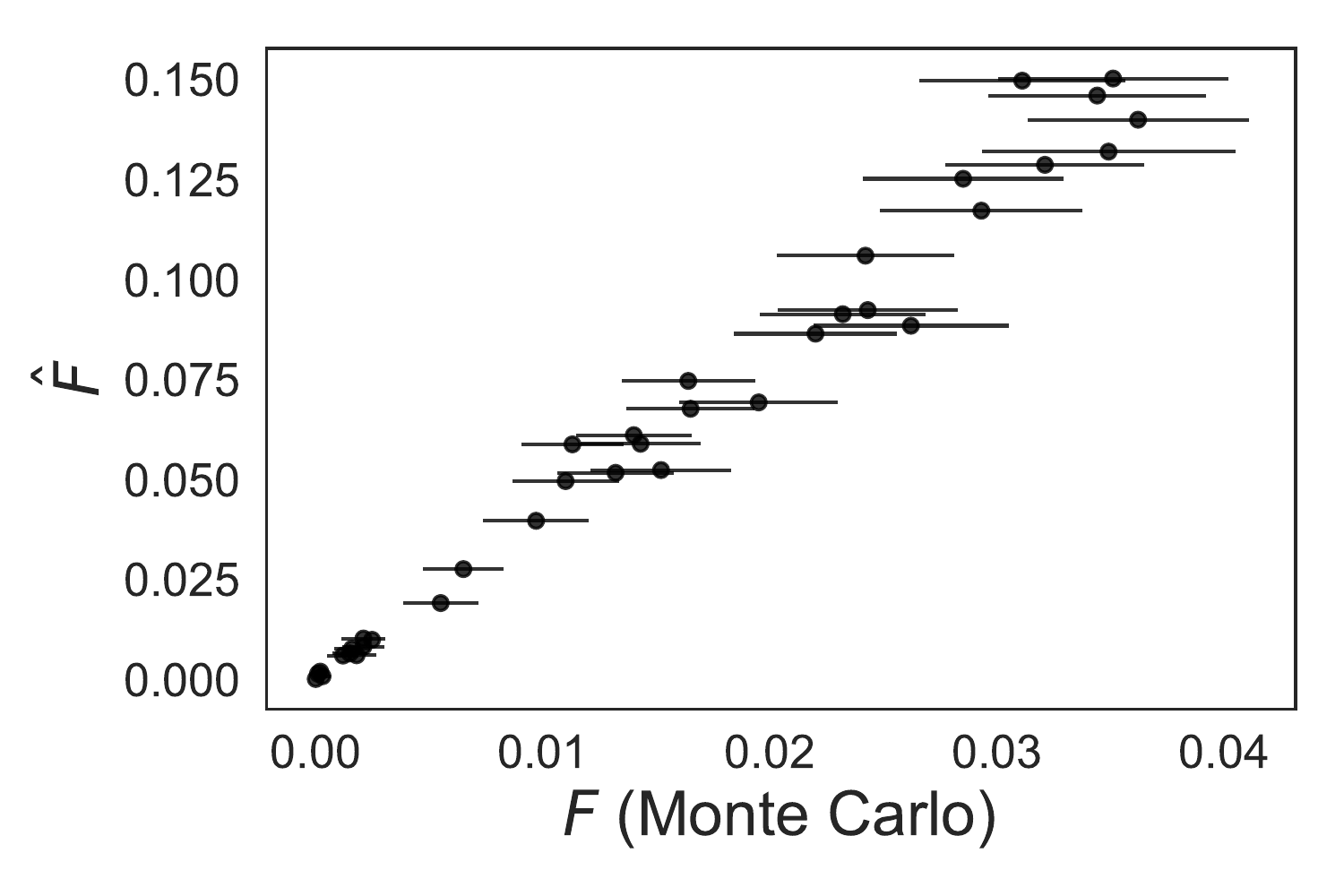}
      \caption{$\batchUtilApprox$ vs. $\batchUtil$ for GB1}
      \label{fig:gb1-comp}
    }
  \end{subfigure}
  \begin{subfigure}[b]{.35\textwidth}
    \centering
    {
      \includegraphics[trim={0pt 0pt 0pt 0pt}, width=\textwidth]{./fig/objectives_comparison}
      \caption{$\batchUtilApprox$ vs. $\batchUtil$ for PhoQ}
      \label{fig:phoq-comp}
    }
  \end{subfigure}
  \caption{Comparing $\batchUtilApprox$ (Eq.~\eqref{eq:batch-util-approx}) against the Monte Carlo estimates of $\batchUtil$ (Eq.~\eqref{eq:batch-util}). Error bars are standard errors for the Monte Carlo estimates. The approximate objective correlates well with Monte Carlo estimates of the exact objective.}
  \label{fig:exp:comparison}
\end{figure}

}{}
\end{document}